\documentclass[letterpaper]{article}
\usepackage{aaai2starai}
\usepackage{times}
\usepackage{helvet}
\usepackage{courier}
\usepackage{general}
\usepackage{graphicx}
\usepackage[linesnumbered]{algorithm2e}
\usepackage{bbm}

\newcommand{\OmegaSn}{\Omega^{(S,n)}}
\newcommand{\Omegan}{\Omega^{(n)}}
\newcommand{\Omegam}{\Omega^{(m)}}
\newcommand{\Omegak}{\Omega^{(k)}}
\newcommand{\Qn}{Q^{(n)}}
\newcommand{\Qm}{Q^{(m)}}

\newcommand{\Ln}{L^{(n)}}
\newcommand{\Lm}{L^{(m)}}

\newcommand{\indicator}{\mathbbm 1}

\newcommand{\true}{\emph{true}}
\newcommand{\false}{\emph{false}}

\begin{document}
%
\title{Inference, Learning, and Population Size: Projectivity for SRL Models}
\author{Manfred Jaeger\\
Aalborg University\\
jaeger@cs.aau.dk
\And
Oliver Schulte\\
Simon Fraser University\\
oschulte@cs.sfu.ca
}
\maketitle
\begin{abstract}
\begin{quote} A subtle difference between propositional and relational data is that in many relational models, marginal probabilities depend on the population or domain size. This paper connects the dependence on population size to the classic notion of projectivity from statistical theory: Projectivity implies that relational predictions are robust with respect to changes in domain size. We discuss projectivity for a number of common SRL systems, and identify syntactic fragments that are guaranteed to yield projective models. 
The syntactic conditions are restrictive, which suggests that projectivity is difficult to achieve in SRL, and care must be taken when working with different domain sizes.
\end{quote}
\end{abstract}

\section{Introduction} 
In propositional or i.i.d. data, marginal probabilities are independent of the number of entities for which data are available. For example, if we start with a test set of 100 data points with fully observed attributes, and we are told that another 1,000 test cases are available, this information does not change a model's predictions on the original 100, because all data points are independent. Several authors have observed that for relational data, this is not true \cite{lauritzen2017random,shalizi2013consistency,poole2014population,jain2010adaptive}: additional entities, or nodes in a network, are potentially related, and therefore their mere presence can be inferentially relevant.

The dependence on domain size raises several difficulties. (1) Counter-intuitiveness: if we are interested in making predictions about the 100 members of a Facebook group, learning that Facebook has gained another 1,000 users in the last hour should not change our predictions about the group of interest, unless we have specific information about the new users. (2) Complexity of Inference: We typically learn from fairly large networks and apply the results of learning to draw inferences about relatively small sub-networks. If the presence of other nodes is relevant to the marginal probabilities associated with the sub-network of interest, inference becomes computationally challenging, as it involves a large summation over the possible states of many nodes outside the sub-network. 
(3) Complexity of Learning: The observed network for learning is typically embedded in a much larger network. If marginal predictions of the observed sub-network depend on the size of the embedding network, learning should not treat the sub-network as a closed world to be analyzed in isolation. 

Our aim in this paper is to connect dependence on domain size with the classic concept of {\em projectivity} from statistical theory \cite{shalizi2013consistency}. An SRL model typically defines a family of probability distributions over relational structures, one for each domain size $n$. Projectivity requires that  the distribution for size $m$ agrees with the marginal distribution obtained from any large domain $n > m$. This ensures that the distributions for different domain sizes are consistent with each other, and that predictions about a sub-network do not depend on the presence of outside nodes.

Most (if not all) SRL modeling frameworks support the construction of non-projective models, and in many cases, this is an important source of model expressivity. However, taking into consideration the significant benefits in terms of computational complexity and robustness of inference and learning results, it is worthwhile to explore projective fragments of SRL languages. After establishing the formal framework for defining and analyzing projectivity in an SRL context (Sections~\ref{sec:background},\ref{sec:projectivity}), we identify syntactic constraints for three prototypical SRL frameworks that ensure projectivity (Section~\ref{sec:fragments}). We finally take a closer look at how projectivity can justify maximum likelihood learning from observed sub-networks (Section~\ref{sec:learning}).  


\paragraph{Related Work.} Several SRL works have addressed changing domain sizes. \citeauthor{poole2014population} discuss the impact of domain size and inference~\citeyear{poole2014population}.  
Convergence of inference~\cite{jaeger1998convergence} and learning~\cite{Xiang2011} in the domain size limit has been investigated.~\citeauthor{kuzelka2018relational} discuss the validity of learning from a sub-network embedded in a larger network. They focus on models that satisfy a given set of marginal probabilities, not projectivity.  Statisticians have examined the projectivity of network models from the exponential family \cite{shalizi2013consistency,lauritzen2017random}, but have not considered other SRL models.

\section{Background}
\label{sec:background}

The following definitions provide a framework for talking about SRL systems in general.

A relational \emph{signature} $S$\ contains relation symbols of varying arities. We write 
$r/k$ for a relation $r$ of arity $k\geq 0$.  
A \emph{possible world} $\omega$\ (for $S$)  is given by a finite domain $D=\{d_1,\ldots,d_n\}$\ on 
which extensions of the relations in $S$\ are defined. We then refer to $n$ as the \emph{size} of 
$\omega$. Moreover, we generally assume that $D=\{0,\ldots,n-1\}=:[n]$. Then $\OmegaSn$ denotes the set of 
all possible worlds for the signature $S$ over the domain $[n]$. 
This can be abbreviated as  $\Omegan$ when the signature is understood from the context. 

$\omega\in \Omegan$ can also be identified with a truth value assignment 
\begin{displaymath}
  \omega: r(\boldi) \mapsto \{\true,\false\}
\end{displaymath}
for all ground atoms $r(\boldi)$, where $r/k\in S$ and $\boldi=(i_1,\ldots,i_k)\in [n]^k$.

We denote with $\Delta\Omegan$ the set of all probability distributions on $\Omegan$. 
A \emph{random relational structure model (rrsm)} is a family of distributions 
$\{\Qn \in\Delta\Omegan \mid n\in\Nset\}$. The term RRSM, which was first
introduced in~\cite{JaegerMI02}, is intended to emphasize that this is a multi-relational generalization of
the classical concept of random graph models (or \emph{random networks}~\cite{lauritzen2017random}).  

For this paper we limit ourselves to the consideration of \emph{exchangeable} RRSMs: a single distribution
$\Qn$ is exchangeable, if $\Qn(\omega)=\Qn(\omega')$ whenever $\omega$ and $\omega'$ are isomorphic. An 
RRSM is exchangeable, when all $\Qn$ ($n\in\Nset$) are exchangeable. 
Exchangeability of $\Qn$ implies that 
$\Qn(r(\boldi)=\true)=\Qn(r(\boldj)=\true)$ when $\boldj = \pi(\boldi)$ for some permutation $\pi$ of $[n]$.

RRSMs defined by typical SRL frameworks are necessarily exchangeable when the model specification does not
make use of any constants referring to specific domain elements, and is not conditioned on a pre-defined
structure on the domain. Examples of RRSMs that do not satisfy this condition are temporal models such as
hidden Markov models, where the model is conditioned on a linear ordering $0<1<\cdots< n-1$ of the domain, 
and the marginal distributions at different timepoints $t\neq t'$ are ususally different.

\subsection{SRL Systems}

In this section we briefly review three different SRL frameworks  that can be used to define
RRSMs. We do not give complete definitions of syntax and semantics for these frameworks here, but
only describe their main characteristics by way of examples. The selected systems are representatives
for three distinct modeling paradigms.

\subsubsection{Relational Bayesian Networks}

Relational Bayesian networks (RBNs)~\cite{Jaeger97UAI} are a representative of frameworks that
are closely linked to directed graphical models. Other approaches in this group include 
probabilistic relational models~\cite{friedman1999learning}, Bayesian logic programs~\cite{kersting2000interpreting}
and relational logistic regression~\cite{kazemi2014relational}. 

An RBN associates with each relation $r\in S$ a functional expression that defines the probabilities
of ground atoms $r(\boldi)$ conditional on other ground atoms. Functional expressions can be formed 
by combining and nesting basic constructs, which support Boolean conditions, mixtures of distributions, 
and aggregation of dependencies using the central tool of \emph{combination functions}. 

For example, the two formulas 

\begin{displaymath}
  \begin{array}{lll}
    \emph{red}(X) & \leftarrow  & 0.3 \\
    \emph{black}(X) & \leftarrow  & {\tt if}\ \emph{red}(X):0\ {\tt else}:\ 0.5
  \end{array}
\end{displaymath}
define Boolean attributes \emph{red} and \emph{black}, such that no element can be both \emph{red}
and \emph{black}. One can then condition a binary \emph{edge} relation on the colors of the nodes:
\begin{displaymath}
  \begin{array}{lll}
    \emph{edge}(X,Y) & \leftarrow &  {\tt if}\  \emph{red}(X)\, \&\, \emph{red}(Y): 0.7 \\
    & & {\tt else\ if}\  \emph{black}(X)\, \&\, \emph{black}(Y): 0.4 \\
    & & {\tt else}:\ 0.05 \\
  \end{array}
\end{displaymath}

Together, the three formulas for \emph{red}, \emph{black} and \emph{edge} represent a simple 
stochastic block model~\cite{snijders1997estimation}. 
The main expressive power of RBNs derives from \emph{combination functions} that 
can condition a ground atom on properties of other domain entities:
\begin{displaymath}
  \begin{array}{lll}
   t(X)  & \leftarrow & \emph{noisy-or}\{ {\tt if}\ \emph{edge}(X,Y)\, \&\, \emph{red}(Y):0.2  \mid Y  \}
  \end{array}
\end{displaymath}
This formula makes the probability of the attribute $t$ for $X$ dependent on the number of \emph{red}
entities $Y$ that $X$ is connected to via an \emph{edge}. Each such entity causes $t(X)$ to be
true with probability 0.2, and the different causes $Y$ are modeled as independent via 
the \emph{noisy-or} combination function.

An RBN specification defines an RRSM, provided that for each $n$ the formulas induce an 
acyclic dependency structure on the atoms of $\Omega^{(S,n)}$. The distribution $\Qn$ can then be 
represented by a Bayesian network whose nodes are the ground atoms $r(\boldi)$, and where 
$r'(\boldi')$ is a parent of $r(\boldi)$, if the truth value of $r'(\boldi')$ is needed to 
evaluate the probability formula for $r(\boldi)$. 

\subsubsection{Markov Logic Networks}

Markov Logic networks (MLNs)~\cite{RicDom06} are the multi-relational generalization of exponential 
random graph models~\cite{frank1986markov,wasserman1996logit}. An MLN consists of a set
of weighted formulas. As an example, let $S=\{\emph{red}/1,\emph{edge}/2\}$, and define
the following two weighted formulas:
\begin{displaymath}
  \begin{array}{lr}
    \phi_1(X,Y) :\equiv \emph{edge}(X,Y)\wedge \emph{red}(X) \wedge \emph{red}(Y) & 1.2 \\
    \phi_2(X,Y) :\equiv \emph{edge}(X,Y)\wedge \emph{red}(X) \wedge \neg \emph{red}(Y) & -0.2 \\
  \end{array}
\end{displaymath}
These two weighted formulas represent a homophily model for the \emph{red} attribute 
in a graph. Each pair $(i,j)\in [n]^2$ contributes a weight of 1.2 (-0.2) to a possible 
world $\omega\in\OmegaSn$ if $\phi_1(i,j)$ ($\phi_2(i,j)$) is true in $\omega$. The probability
$\Qn(\omega)$ then is the normalized sum of all weights contributed by all the weighted formulas, 
and all possible substitutions of domain elements for the variables in the formulas.

$\Qn$ can be represented by a Markov network whose nodes are the ground atoms $r(\boldi)$, and
where two atoms $r(\boldi),r'(\boldi')$ are connected by an edge if these two atoms 
appear jointly in a grounding of one of the weighted formulas.

\subsubsection{ProbLog}

ProbLog~\cite{DeRKimToi2007,kimmig2011implementation} is a representative of RRSMs that 
are closely linked to logic programming. Other frameworks in this class are Prism~\cite{Sato95}, and
independent choice logic~\cite{Poole08}. A ProbLog model consists of 
a set of  \emph{labeled facts}, which are atoms with a probability label attached:
\begin{equation}
\label{eq:problog1}
    0.8::\ \emph{red}(X)
\end{equation}
together with \emph{background knowledge} consisting of (non-probabilistic) definite clauses:
\begin{equation}
\label{eq:problog2}
    \emph{edge}(X,Y) :-\ \emph{red}(X),\emph{red}(Y).
\end{equation}
We note that~\cite{kimmig2011implementation} emphasize the use of ground atoms in the labeled 
facts. Since our interest is with  generic, domain-independent RRSMs, we restrict attention
to ProbLog models that do not contain domain constants.

The ProbLog model consisting of (\ref{eq:problog1}) and (\ref{eq:problog2}) defines the 
probability of $\omega\in\OmegaSn$ as 0 if $\omega$ does not satisfy the property that
two elements are connected by an edge if and only if they both are red (i.e., 
$\omega$ must be a minimal model of  (\ref{eq:problog2})). If $\omega$ satisfies 
 (\ref{eq:problog2}), then its probability is $0.8^{|\emph{red}(\omega)|}0.2^{n-|\emph{red}(\omega)|} $, 
where $\emph{red}(\omega)$ denotes the set of elements  $i\in [n]$ for which 
$\emph{red}(i)$\ is true in $\omega$.   

In order to also make rules such as (\ref{eq:problog2}) probabilistic, one can introduce
latent relations that represent whether rules in the background knowledge become applicable.
In our example, we can add a new relation \emph{rule}/2,  an additional labeled fact
\begin{equation}
\label{eq:problog3}
    0.5::\ \emph{rule}(X,Y),
\end{equation} 
and modify the background knowledge to 
\begin{equation}
\label{eq:problog4}
    \emph{edge}(X,Y) :-\ \emph{red}(X),\emph{red}(Y),\emph{rule}(X,Y).
\end{equation}
Now the resulting ProbLog model is a (partial) stochastic block model, where with probability
0.5 two red nodes are connected by an edge (and still no other than pairs of red nodes could 
be connected; further labeled facts and rules can be added to also model connections between
non-red nodes).

\section{Projectivity}
\label{sec:projectivity}

For
$\Qn\in\Delta\Omegan$  and a subset $\{i_0,\ldots,i_{m-1}\}\subseteq [n]$ we 
denote with  $\Qn\downarrow  \{i_0,\ldots,i_{m-1}\}$ the marginal distribution on the sub-structures
induced by $\{i_0,\ldots,i_{m-1}\}$, or, equivalently, the marginal distribution on the ground atoms
$r(\boldi)$ with $\boldi\in  \{i_0,\ldots,i_{m-1}\}^k$. When $\Qn$ is exchangeable, then the induced
distribution is independent of the actual choice of the elements $i_h$, and we may assume that
$\{i_0,\ldots,i_{m-1}\}=\{0,\ldots,m-1\}$. We therefore can limit attention to the 
marginalizations
\begin{equation}
  \label{eq:marginalizeQn}
  \Qn\downarrow [m] \in \Delta \Omegam.
\end{equation}

Based on \cite{shalizi2013consistency} we define several different versions of ``projectivity'' for 
probabilistic relational models. Our definitions are more restricted than the one given by Shalizi 
and Rinaldo~(\citeyear{shalizi2013consistency}) in that 
we specifically tailor the definitions to  random relational structure models.
For convenience, we also include the condition of exchangeability in the following 
definitions, even though exchangeability and projectivity are not necessarily linked.

\begin{definition}
\label{def:proj1}
  A RRSM is \emph{projective}, if 
  \begin{itemize}
  \item every $\Qn$ is exchangeable
  \item for every $n$, every $m<n$:
    $\Qn\downarrow [m]=Q^{(m)}$.
  \end{itemize}
\end{definition}

\begin{example}
  A simple classical example for a projective RRSM is the Erd\"os-R\'{e}nyi random graph model with
edge probability $p$, where
$S=\{e/2\}$, and for $\omega\in\Omegan$: $\Qn(\omega)=p^{|e(\omega)|}(1-p)^{n^2-|e(\omega)|}$, 
where $ |e(\omega)| $ denotes the number of edges in $\omega$.

An example of a very different nature is the RRSM where for every $n$: 
\begin{equation}
  \label{eq:clique_empty}
  \Qn= \frac{1}{2}\indicator_{K_n} +  \frac{1}{2}\indicator_{E_n}
\end{equation}
 where $K_n$\ is the complete and $E_n$\ the empty graph over $n$, and 
$\indicator_{\omega}$ denotes the distribution that assigns probability one to $\omega$. 

Not projective are sparse 
random graph models, where the edge probability is a decreasing function of the domain size $n$, e.g. 
$\Qn(e(i,j))=\theta/n$\ for some parameter $\theta$. 
\end{example}

Definition~\ref{def:proj1} relates to a single RRSM. In the context of learning, we are initially only 
given a space of candidate models. The models in this space are typically characterized by a structural 
component, and a set of numerical parameters. In the following, we assume that the learning problem only
consists of optimizing over a fixed set of numerical parameters, i.e., structure learning is outside 
the scope of our considerations. 

Concretely, let $\Theta\subseteq \Rset^k$ be a parameter space, so that each $\boldtheta\in\Theta$ defines
a distribution $\Qn_{\boldtheta}\in\Delta\Omegan$. Then  $\{\Qn_{\boldtheta}\mid \boldtheta\in\Theta, n\in\Nset\}$
is a  \emph{parametric family of RRSMs}.

This leads to the following  definition of projectivity, which is essentially the one
of~\cite{shalizi2013consistency}.
\begin{definition}
\label{def:projective}
  A parametric family of RRSMs is \emph{projective}, if  $\{\Qn_{\boldtheta}\mid n\in\Nset\}$ is
projective for every $\boldtheta\in\Theta$.
\end{definition}

We can weaken this as follows:
\begin{definition}
   A parametric family of rrms is \emph{structurally projective} if 
   \begin{itemize}
   \item  every $\Qn_{\boldtheta}$ is exchangeable
   \item for every $n,\boldtheta$, and $m<n$:  there 
exists $\boldtheta'\in \Theta$ such that  $\Qn_{\boldtheta}\downarrow [m] =Q^{(m)}_{\boldtheta'}$.
   \end{itemize}
\end{definition}

This concept of structural projectivity corresponds to  \emph{weak consistency} in the sense of 
\cite{lauritzen2017random} restricted to exchangeable distributions.

An example for structurally projective models are sparse random graph models with edge probabilities 
$\theta/n$: here we then have $\Qn_{\theta}\downarrow I=Q^{(m)}_{\theta'}$ with $\theta'=\theta m/n$

\begin{example}
\label{ex:mlnnotsp}
  MLNs are not structurally projective: consider the MLN
  \begin{displaymath}
    \begin{array}{ll}
      a(X),e(X,Y) & w
    \end{array}
  \end{displaymath}
defining probability distributions $\Qn_w$ depending on the weight parameter $w$.
Consider the  conditional probability 
\begin{equation}
\label{eq:qna0}
  q(n,w):=\Qn_w(a(0)| \neg e(0,0),\neg e(0,1),\neg e(1,0),\neg e(1,1)).
\end{equation}
If $w>0$, then $q(n,w)$ is increasing in $n$ (and in $w$). However, for 
$n=2$, we have $q(n,w)=1/2$, regardless of the value of $w$. Thus, for a suitable combination
of $n,w$ where $q(n,w)>1/2$, we cannot find a value $w'$ such that
$\Qn_q\downarrow [2] = Q^{(2)}_{w'}$.
\end{example}

The lack of structural projectivity sets some limits to the approach of defining the
weights in an MLN as functions of the domain size $n$~\cite{jain2010adaptive} in 
order to compensate for the domain dependence of the model. 

\begin{example}
  RBNs are not structurally projective: similarly to Example~\ref{ex:mlnnotsp}, we can construct
a counterexample as follows. Consider the RBN
\begin{displaymath}
  \begin{array}{lll}
   \emph{edge}(X,Y) & \leftarrow & 0.5 \\
   a(X)  & \leftarrow & \emph{noisy-or}\{ {\tt if}\ \emph{edge}(X,Y):\theta  \mid Y  \}
  \end{array}
\end{displaymath}
defining probability distributions $\Qn_{\theta}$ depending on the probability parameter $\theta$.
Defining $q(n,\theta)$ as in (\ref{eq:qna0}), we obtain that $q(2,\theta)=0$, regardless of $\theta$,
but $q(n,\theta)>0$ for $n\geq 3$ and $\theta>0$. 
\end{example}

\section{Projective Fragments of SRL Systems}
\label{sec:fragments}

In this section we identify for the three representative SRL frameworks introduced in 
Section~\ref{sec:background}, restrictions on model structure that give rise to projective models. 

For the propositions stated in this section we give short proof sketches. Full proofs would
need to refer to complete formal specifications of syntax and semantics of the various 
frameworks, which we have omitted in Section~\ref{sec:background}. Even the proof sketches we
provide may appeal to some facts about the individual frameworks that were not explicitly
spelled out in Section~\ref{sec:background}.


\begin{proposition}
  An RBN defines a projective RRSM if it does not contain any combination functions.
\end{proposition}

\begin{proof}[Sketch]
Consider the Bayesian network representation $B^{(n)}$ 
of $\Qn$, and a parent child pair $ r'(\boldi') \rightarrow r(\boldi)$ in $B^{(n)}$.
If the probability for the relation $r$ is defined without the use of combination functions, 
then all constants appearing in $\boldi'$ must also be contained in $\boldi$. This implies 
that  the sub-network for the ground atoms over $[m]$ ($m\leq n$) is 
an upward-closed sub-graph of $B^{(n)}$ whose structure and probability parameters do not
depend on $n$. This means that the marginal distributions on the ground atoms over  $[m]$
does not depend on $n$.
\end{proof}

Even though combination functions are the main source for the expressive power of RBNs, one
can still encode some relevant models with the combination function free fragment. 
One example are the stochastic
block models as shown in Section~\ref{sec:background}. Another example are temporal models 
such as dynamic Bayesian networks or hidden Markov models. However, as mentioned in Section~\ref{sec:background}, these models are not exchangeable, and therefore outside the scope of this paper.


\begin{proposition}
\label{prop:mlnfrag}
An MLN  defines a projective RRSM  if its formulas $\phi_i$ satisfy the property that
any two atoms appearing in $\phi_i$ contain exactly the same variables. 
\end{proposition}

\begin{proof}[Sketch]
If the condition of the proposition holds, then 
the Markov network representation $M^{(n)}$ of $\Qn$  decomposes  
into a system of disconnected sub-networks, where each sub-network contains ground
atoms over a fixed set of domain elements, with a cardinality at most equal to the maximal 
arity of any $r\in S$.  The structure and parameterization of sub-networks containing
only the atoms for domain elements from $[m]$  ($m\leq n$)  are the same for all $n$, and they define a 
marginal distribution that is independent of the nodes contained in the other sub-networks, i.e., 
independent of $n$. 
\end{proof}

Alternatively, Proposition~\ref{prop:mlnfrag} can also be proven by an application of 
Theorem 1 of~\cite{shalizi2013consistency}.
Our MLN example from Section~\ref{sec:background} does not satisfy the restriction of 
Proposition~\ref{prop:mlnfrag}. A somewhat synthetic example that satisfies the condition is

\begin{displaymath}
  \begin{array}{lr}
   \emph{red}(X)\wedge \emph{edge}(X,X)  & -1.5 \\
   \emph{edge}(X,Y)\wedge\emph{edge}(Y,X) & 0.8 \\ 
  \end{array}
\end{displaymath}

This MLN represents a model according to which red nodes are unlikely to have self-loops, and
the edge relation tends to be symmetric. It is an open question whether our projective MLN fragment
contains more natural and practically relevant classes of models.


\begin{proposition}
\label{prop:problogproj}
  A ProbLog model defines a projective RRSM, if all the background knowledge clauses satisfy 
the property that the body of the clause does not contain any variables that are not contained in the
head of the clause.
\end{proposition}

\begin{proof}[Sketch]
If the stated property holds, then a proof for a ground atom $r(\boldi)$ can only contain
ground atoms $r'(\boldi')$ with $\boldi'\subseteq \boldi$. The probability that $r(\boldi)$ is
provable when $\boldi\subset [m]$ then only depends on the probabilities of ground labeled facts
with arguments from $[m]$. The joint distribution of these ground facts is independent of $n$. 
\end{proof}

The ProbLog example of Section~\ref{sec:background} satisfies the condition of 
Proposition~\ref{prop:problogproj}. 

\subsection{Discussion}
Our conditions for the projective RBN and ProbLog structures are very
similar, and it seems that both
fragments support more or less the same types of models. 
The projectivity conditions are essentially limitations on probabilistic dependencies. For
the frameworks that (implicitly) encode a directed sampling process for possible worlds (RBN, ProbLog), 
the limit on the dependency is that when a ground atom $r(\boldi)$ is sampled, its distribution 
cannot depend on ground atoms containing elements other than included in $\boldi$. However, 
$r(\boldi)$, in turn, may influence the sampling probabilities of other atoms containing elements 
not in $\boldi$. As a consequence, both in the RBN and the ProbLog model of Section~\ref{sec:background},
the value of $\emph{red}(i)$ influences the probability for $\emph{edge}(i,j)$, and as a result,
the two variables $\emph{red}(i), \emph{edge}(i,j)$ are not independent. 

Due to the undirected nature of MLNs, one there cannot impose an ``upstream only'' limitation on
probabilistic dependencies. As a result, the restriction imposed in Proposition~\ref{prop:mlnfrag}
implies that a when a model defines two random variables  $\emph{red}(i), \emph{edge}(i,j)$ ($i\neq j$),
these  must be independent.

\section{Projectivity and Inference}
\label{sec:inference}

In an inference scenario, we are given an RRSM $\{\Qn\mid n\in \Nset\}$, a domain of size $n$, and
a query, which for the sake of concreteness we assume to be of the form
\begin{displaymath}
  P(r(\boldi)\mid (\neg)r_1(\boldi_1),\ldots,(\neg)r_h(\boldi_h)) = ?,
\end{displaymath}
where the $(\neg)r_j(\boldi_j)$ are observed evidence literals. Let $I:=\boldi\cup\cup_{j=1}^h \boldi_j$, 
and $m:=|I|$. 
To answer the query, all we need is the marginal distribution 
$\Qn\downarrow I$. If the model is projective, then this marginal is equivalent to $\Qm$, and 
independent of $n$. In practice this means, that when inference is performed by grounding the
relational model over a concrete domain, that we here only need to ground the model over the
domain that contains exactly the entities mentioned in the query. 

Apart from the clear computational advantages that projectivity affords, it also leads to 
robustness with regard to domains that are only  incompletely known or observed: in many cases
it can be difficult to specify exactly the size of the domain in which our observed entities $I$ ``live''.
(what is the domain of a person's social network?). Here projectivity means that our inferences
are not impacted by missing or imprecise information about the domain.

\section{Projectivity and Learning}
\label{sec:learning}

In a learning setting, we are interested in estimating the
parameter $\boldtheta$ for a given parametric family of RRSMs. We assume that
the training data consists of one or several observed possible worlds, in the latter 
case possibly worlds of varying sizes. Like~\cite{Xiang2011} and
\cite{kuzelka2018relational} we mostly focus on the scenario where the training data 
consists of a single possible world $\omega\in\Omegam$, and we estimate $\boldtheta$ by maximizing
the likelihood

\begin{equation}
  \label{eq:thetalik}
  L(\boldtheta|\omega) = \Qm_{\boldtheta}(\omega).
\end{equation}
For simplicity we here restrict attention to pure maximum likelihood inference, but our considerations
are also pertinent for penalized likelihood or Bayesian inference approaches.

For this learning problems, too, the dependence on the domain size $n$ is an important concern.
Consider, for example, the problem of learning a model for an evolving (social) network. 
Is the model we learn today, while the network is of size $m$, still a good model when the
network has evolved to size $n>m$? Moreover, the data from which we learn may not be  complete,
and not contain the data about all the entities that are actually present in the relational 
domain. Then we try to learn a model for a domain of size $n$, from data corresponding to 
a domain of smaller size $m<n$~\cite{kuzelka2018relational}.

The general question we want to address, therefore,  is the following: 
let $\boldtheta^*\in\Theta$ be the parameter estimate we obtain by maximizing
(\ref{eq:thetalik}) for some $\omega\in\Omegam$. Let 
$\omega'\in\Omegan$ such that $\omega'\downarrow I = \omega$ for 
some $I\subset [n]$ of size $m$, and $\boldtheta^{**}$ the estimate obtained 
by maximizing $\Qn_{\boldtheta}(\omega')$. What can we say about the relationship 
between $\boldtheta^*$ and $\boldtheta^{**}$?

To obtain more precise statements of this overall question, we first consider the scenario where we do
not know $\omega'$, but we do know $\omega'$'s size $n$. Then, given the partial observation
$\omega$, we can maximize likelihood  in the appropriate space 
$\{\Qn_{\boldtheta}| \boldtheta\in\Theta\}$ via the marginal likelihood function 
\begin{equation}
  \label{eq:thetafvlik}
  M\Ln(\boldtheta|\omega):=\Qn_{\boldtheta}( \{ \omega'\in\Omegan: \omega' \downarrow [m]=  \omega \}) 
\end{equation}

The following then is immediate:
\begin{proposition}
If the parametric family $\{\Qn_{\boldtheta}\mid \boldtheta\in\Theta, n\in \Nset\}$ is projective, then
the likelihood functions (\ref{eq:thetalik}) and (\ref{eq:thetafvlik}) are identical.
\end{proposition}


Even in the context of a projective family, however, learning from $\omega$ that is a sub-sample 
of the ``true'' world $\omega'$ usually is 
problematic: the use of the marginal likelihood function
(\ref{eq:thetafvlik}) only is justified, when the data in $\omega$ is 
\emph{missing at random}~\cite{Rubin76}. This will be 
the case when $\omega$ is the induced sub-structure of  $m$  uniformly, randomly selected elements 
from the domain $[n]$ 
of $\omega'$ (\emph{induced subgraph sampling}~\cite[Chapter 5]{kolaczyk2009statistical}). However, a more
realistic scenario for observing a substructure of a possible world $\omega'$ is by mechanisms such as
\emph{traceroute} or \emph{snowball} sampling~\cite[Chapter 5]{kolaczyk2009statistical}). 
When $\omega$ is observed through such a mechanism, then
(\ref{eq:thetafvlik}) is not an appropriate likelihood function.

\subsection{Projectivity and Maximum Likelihood Learning}
In the following, we make the 
idealized assumption
that $\omega\in\Omegam$ was observed through induced subgraph sampling
from a larger world $\omega'\in\Omegan$.
In conjunction with projectivity,  the use of 
(\ref{eq:thetalik}) then is justified, giving us the parameter estimate $\theta^*$. 
Can we derive some guarantees for the quality of $\theta^*$ as an approximation
of the estimate $\theta^{**}$ one would obtain by maximizing 
$\Qn_{\boldtheta}(\omega')$?

Clearly, for any specific  pair $\omega,\omega'$,  no
guarantees can be given, because $\omega$ could be a very un-representative 
sub-structure of $\omega'$. We therefore can only ask whether such guarantees can
be given in expectation, where  expectation is
with respect to induced subgraph sampling of $\omega$ from $\omega'$. 

This question still has two 
distinct precise formalizations, expressed by the following two equalities:

\begin{eqnarray}
  E_{\omega}  [ \argmax_{\boldtheta} \emph{log}L^{(m)}(\boldtheta|\omega)  ] 
 =  \argmax_{\boldtheta} \emph{log}L^{(n)}(\boldtheta|\omega') & & 
\label{eq:unbiased1}\\
  \argmax_{\boldtheta} E_{\omega} [ \emph{log}L^{(m)}(\boldtheta|\omega)]  
 =  \argmax_{\boldtheta} \emph{log}L^{(n)}(\boldtheta|\omega') & &
\label{eq:unbiased2}
\end{eqnarray}

We here turn to the log-likelihood, because  the expectation in the left-hand side
expression of  (\ref{eq:unbiased2}) is more meaningful when applied to the log-likelihood.
For all other expressions in  (\ref{eq:unbiased1}) and (\ref{eq:unbiased2}) the application
of the \emph{log} makes no difference. In cases where the log-likelihood functions need 
not have a unique maximum, the $\argmax_{\boldtheta}$ should be read as returning
the set of all maxima.

Condition (\ref{eq:unbiased1}) basically says that estimating $\boldtheta$ from an $m$-element induced 
substructure of $\omega'$ is an \emph{unbiased} estimator for the estimate one would have obtained from 
$\omega'$. Condition (\ref{eq:unbiased2}) essentially expresses a statistical \emph{consistency} property: if one
takes repeated size $m$ samples $\omega_1,\ldots,\omega_N$, and maximizes the sample log-likelihood
$\frac{1}{N}\sum_{i=1}^N   \emph{log}\Lm(\boldtheta|\omega_i)$ then, for large $N$, this will become equivalent
to maximizing $\Ln(\boldtheta|\omega')$. We note though, that (\ref{eq:unbiased2}) alone does not
directly guarantee this consistency. Additional regularity conditions on the likelihood function will be 
needed. 
 
The following is a cautionary example that shows that even for projective models, neither 
(\ref{eq:unbiased1}) nor (\ref{eq:unbiased2}) need hold.

\begin{example}
Consider the following RBN:
  \begin{eqnarray}
    \emph{red}(X) & \leftarrow  & \theta \label{eq:rbncounterex1} \\
    \emph{edge}(X,Y) & \leftarrow  & \theta \label{eq:rbncounterex2}
  \end{eqnarray}
In this RBN there is a common parameter $\theta$ that denotes both the probability for the
attribute $a/1$, and the edge relation $\emph{edge}/2$.   This form of parameter sharing by different
probability formulas is supported by RBNs and their learning tools, even though it seems to be
of limited use in practice. This RBN clearly is projective. 

Now consider the possible world $\omega'\in\Omegan$ where $a(i)$ is true for $i=0,\ldots,n/2-1$, and
false for $i=n/2,\ldots,n-1$, and where $\emph{edge}(i,j)$ is false for all $i,j$. Let $m=2$. 
A random $\omega\in\Omega^{(2)}$ drawn from $\omega'$ then is with equal probability one of the 4 
worlds with $(r(0),r(1))$, $(r(0),\neg r(1))$, $(\neg r(0),r(1))$, or $(\neg r(0),\neg r(1))$ (and 
no edges). Writing $l\theta$ for $\emph{log}\theta$, and $l\bar{\theta}$ for 
$\emph{log}(1-\theta)$, we obtain as the expected log-likelihood:
\begin{multline}
  E_{\omega} [ \emph{log}L^{(2)}(\boldtheta|\omega)]= \\
 \frac{1}{4}
((4l\bar{\theta}+2l\theta)+ 2(4l\bar{\theta}+l\theta +l\bar{\theta})+(4l\bar{\theta}+2l\bar{\theta}))
\label{eq:lik1}
\end{multline}
where the  recurring first term $4l\bar{\theta}$ accounts for the absence of the four possible edges (including
possible self-loops) in all possible sampled worlds. Maximizing this gives
\begin{displaymath}
  \theta^*=\argmax_{\boldtheta} E_{\omega} [ \emph{log}L^{(2)}(\boldtheta|\omega)] = 1/6.
\end{displaymath}
The log-likelihood function induced by $\omega'$ is
\begin{equation}
\label{eq:lik2}
  \emph{log}L^{(n)}(\boldtheta|\omega') = n^2 l\bar{\theta} + (n/2)l\theta +  (n/2)\bar{\theta}  
\end{equation}
hich is maximized by
\begin{displaymath}
  \theta^{**}=\argmax_{\boldtheta} \emph{log} L^{(n)}(\boldtheta|\omega') = \frac{n/2}{n^2+n}=\frac{1}{2(n+1)}
\end{displaymath}
Thus, (\ref{eq:unbiased2}) does not hold. We also compute
\begin{displaymath}
   E_{\omega}  [ \argmax_{\boldtheta} L^{(2)}(\boldtheta|\omega)  ] =
   \frac{1}{4}(\frac{1}{3}+ \frac{1}{6}+ \frac{1}{6}+ 0  ) =\frac{1}{6},
\end{displaymath}
and see that (\ref{eq:unbiased1}) also fails (in general, other than in this example,  
the left-hand sides of  (\ref{eq:unbiased1})  and (\ref{eq:unbiased2}) need not be identical).  

Finally, now suppose that the two formulas (\ref{eq:rbncounterex1}) and  (\ref{eq:rbncounterex2}) each
are defined by distinct parameters $\theta_r$ and $\theta_e$, respectively. Then the likelihood function
(\ref{eq:lik1}) becomes
\begin{displaymath}
   \frac{1}{4}
((4l\bar{\theta}_e+2l\theta_r)+ 2(4l\bar{\theta}_e+l\theta_r +l\bar{\theta}_r)+(4l\bar{\theta}_e+2l\bar{\theta}_r)),
\end{displaymath}
and (\ref{eq:lik2}) turns into
\begin{displaymath}
   n^2 l\bar{\theta}_e + (n/2)l\theta_r +  (n/2)\bar{\theta}_r,  
\end{displaymath}
both of which are maximized by $\theta_e=0$ and $\theta_r=1/2$.
\end{example}

We now proceed to establish conditions under which (\ref{eq:unbiased2}) is guaranteed to hold.
We first
restrict attention to models where the sufficient statistics for 
$\Qn_{\boldtheta}(\omega)$ are given by induced substructure counts, defined next.

\begin{definition}
  Let $\omega\in\Omegan$ and $k\leq n$. Let $\tilde{\omega}\in\Omegak$. The 
\emph{ordered substructure count} 
of $\tilde{\omega}$ in $\omega$ is defined as 
\begin{displaymath}
  |\{\boldi: \omega\downarrow \boldi \equiv \tilde{\omega}   \}| =: C_{\tilde{\omega}}(\omega)
\end{displaymath}
where 
\begin{itemize}
\item $\boldi=(i_0,\ldots,i_{k-1})$ ranges over the $\frac{n!}{(n-k)!}$ 
tuples of $k$ distinct elements from $[n]$
\item $\omega\downarrow\boldi$ is the  sub-structure induced by $\boldi$ in $\omega$
\item $ \equiv $ stands for isomorphism under the mapping $i_j \mapsto j$
\end{itemize}
\end{definition}

We refer to this as ordered substructure count, because it corresponds to taking an ordered sample
of $k$ elements from $\omega$, and matching the induced substructure against $\tilde{\omega}$ under
the unique mapping $i_j \mapsto j$ defined by the order of $\boldi$. 
The ordered substructure counts are closely related to the ``Model B'' of~\cite{kuzelka2018relational} for
defining the probability of a formula, and \emph{homomorphism densities} 
for the 
convergence analysis of graph sequences~\cite{borgsa2008convergent}.  

Collecting all ordered substructure counts for worlds $\tilde{\omega}$ up to size $k$, we obtain:

\begin{definition}
  The \emph{complete $k$-count statistics} of $\omega$ is the set
\begin{displaymath}
  C_k(\omega):=\cup_{1\leq l\leq k}\{ C_{\tilde{\omega}}(\omega)\mid \tilde{\omega} \in\Omega^{(l)} \}
\end{displaymath}
\end{definition}

Clearly, there is redundancy in the complete $k$-count statistics, since the substructure counts 
for worlds of size $l<k$ can be derived from the substructure counts for the worlds of 
size $k$. However, it is convenient to also directly include the counts for smaller substructures
explicitly in our statistics. The following definition describes a special case of statistical
sufficiency, which is suitable in our context.

\begin{definition}
  Let $k\geq 1$. A parametric family of RRSMs is \emph{determined by $k$-count statistics},
if for all $n$, $\Qn_{\boldtheta}(\omega)$ depends on $\omega$ only as a function of
$C_k(\omega)$. 
\end{definition}

It is 
straightforward that $k$-count statistics are sufficient for MLNs that contain
at most $k$ variables in each formula. For RBNs,  $k$-count statistics are not sufficient 
in general, but they are sufficient for the restricted class of projective RBNs identified in 
Section~\ref{sec:fragments}. For ProbLog the situation is a bit more involved, as discussed 
in Example~\ref{ex:consistency-problog} below.

As a first step to exploit determination by $k$-count statistics to ensure (\ref{eq:unbiased2}), we 
note that induced substructure sampling provides an unbiased estimator for the
ordered substructure counts of $\omega'$ (normalized to substructure \emph{frequencies}):

\begin{lemma}
\label{lem:suffstats}
  Let $\omega'\in\Omegan$, $k\leq m\leq n$, and $\tilde{\omega}\in\Omegak$. Then
  \begin{displaymath}
    E_{\omega}[\frac{(m-k)!}{m!}  C_{\tilde{\omega}}(\omega)]=\frac{(n-k)!}{n!}C_{\tilde{\omega}}(\omega'),
  \end{displaymath}
where the expectation is with respect to induced substructure sampling from $\omega'$ of 
worlds $\omega\in\Omegam$.
\end{lemma}

Next, we consider a particular form of the likelihood function determined
by $k$-count statistics:

\begin{definition}
\label{def:linsep}
  We say that the log-likelihood function is \emph{linear and separable} in the $k$-count statistics, if
for $\omega\in\Omegam$:
  \begin{equation}
  \label{eq:linsep}
    \emph{log}L^{(m)}(\boldtheta\mid \omega) =
    \sum_{l=1}^k c(m,l) \sum_{\tilde{\omega}\in\Omega^{(l)}}  
    C_{\tilde{\omega}}(\omega)\cdot f_{\tilde{\omega}}(\boldtheta[\tilde{\omega}])
  \end{equation}
where
\begin{itemize}
\item $ c(m,l)$ is a constant depending only on $m$ and $l$
\item $f_{\tilde{\omega}}(\boldtheta[\tilde{\omega}])$ is a function of a subset of 
parameters $\boldtheta[\tilde{\omega}]\subseteq \boldtheta$, such that the following \emph{separability}
property holds: if $\tilde{\omega}\in\Omega^{(l)}$, $\tilde{\omega}'\in\Omega^{(l')}$ with
$l\neq l'$, then $\boldtheta[\tilde{\omega}] \cap  \boldtheta[\tilde{\omega}']=\emptyset$.
\end{itemize}
\end{definition}

\begin{proposition}
\label{prop:linseplik}
  If the likelihood function is linear and separable in the $k$-count
statistics, then (\ref{eq:unbiased2}) holds.
\end{proposition}

\begin{proof}
 Due to the separability condition, the maximization of $\emph{log}L^{(m)}$ can be divided into $k$ 
independent maximization problems of the functions
\begin{multline*}
\emph{log}L^{(m)}_l(\boldtheta\mid \omega) :=
  c(n,l) \sum_{\tilde{\omega}\in\Omega^{(l)}}  
    C_{\tilde{\omega}}(\omega)\cdot f_{\tilde{\omega}}(\boldtheta[\tilde{\omega}])\hspace{5mm}\\
    (l=1,\ldots,k),
\end{multline*}
for disjoint sets of parameters $\boldtheta[l]:=\cup_{\tilde{\omega}\in\Omega^{(l)}}\boldtheta[\tilde{\omega}]$. 

Then with Lemma~\ref{lem:suffstats}:
\begin{multline}
  E_{\omega}[\emph{log}L^{(m)}_l(\boldtheta\mid \omega)]= \\
   c(m,l)\frac{m!}{(m-l)!}  \sum_{\tilde{\omega}\in\Omega^{(l)}}  E_{\omega}[\frac{(m-l)!}{m!} C_{\tilde{\omega}}(\omega)]\cdot f_{\tilde{\omega}}(\boldtheta[\tilde{\omega}])=\\
 c(m,l)\frac{m!}{(m-l)!}  \sum_{\tilde{\omega}\in\Omega^{(l)}}  \frac{(n-l)!}{n!} C_{\tilde{\omega}}(\omega)\cdot f_{\tilde{\omega}}(\boldtheta[\tilde{\omega}])=\\
 \frac{c(m,l)}{c(n,l)}\frac{m! (n-l)!}{(m-l)! n!} \emph{log}L^{(n)}_l(\boldtheta\mid \omega')],
\end{multline}
from which  (\ref{eq:unbiased2}) follows.
\end{proof}

\subsection{Directed Model Examples}
For ProbLog and RBNs we have to consider two different scenarios: the first scenario is that all 
relations that appear in the model are also observed in the possible worlds $\omega$ that constitute our training data. Under complete observability, (\ref{eq:unbiased2})  holds in both models as follows. 

\begin{example}
\label{ex:consistency-rbn}
  Using Proposition~\ref{prop:linseplik} we can show that (\ref{eq:unbiased2}) holds for 
RBNs without combining functions, under the following two additional conditions: 
\begin{description}
\item[(i)] $m \geq k$, where $k$ is  the largest arity
of relations in the underlying signature $S$;
\item[(ii)] the RBN does not contain 
multiple occurrences of the same parameter.
\end{description}
Condition (i) in conjunction with the restriction to the projective fragment 
ensures that the model is determined by $k$-count statistics. Condition (ii) ensures the 
separability property. 
\end{example}

Related to the previous example is Schulte's \citeyear{Schulte2011} pseudo-likelihood function for 
first-order Bayesian networks, which also satisfies the conditions of Definition~\ref{def:linsep}.

\begin{example}
\label{ex:consistency-problog}
Under complete observability,
parameter learning for ProbLog becomes trivial, since
the maximum likelihood parameters for the labeled facts are just obtained as the empirical 
frequencies with which groundings of the relations in the labeled facts are true. In particular, then 
the likelihood function is determined by $k$-count statistics, with $k$ the maximal arity of any
relation in labeled facts, and (\ref{eq:unbiased2})
trivially holds, even without the restriction to projective ProbLog.

However, more realistically, the ProbLog model will contain also some synthetic, unobserved 
relations, as shown in (\ref{eq:problog3}) and (\ref{eq:problog4}).  
In that case, the likelihood function induced by a possible world 
$\omega$ for the observable relations becomes a sum of products of parameters, which does not 
decompose as (\ref{eq:linsep}).
\end{example}

\section{Conclusion} The domain-size dependence of marginal probabilities is a  property that causes difficulties for both inference and learning. This dependence has been examined in statistical theory using the concept of a projective family of distributions, whose inferences do not depend on domain size. This paper considered whether common SRL models define projective families. 
The paper gives sufficient conditions on the model structures for three different types of SRL frameworks that ensure projectivity. The conditions are quite restrictive, which is evidence that projectivity is difficult to achieve in SRL models. For learning, we examined conditions under which maximum likelihood parameter estimation is valid when it is applied to an observed sub-network drawn from a larger network of known size. 

We believe that the projectivity of SRL models is an important and fruitful topic for future research. Open questions include the following. (1) Are there examples of domains for which our MLN projectivity condition is natural (all atoms in a formula share the same first-order variables)? (2) For RBNs, are there conditions on combining rules (rather than model structure) that ensure projectivity? For example, under which conditions does the combining rule \emph{average} define projective models? (3) Do our learning results for maximum likelihood learning carry over to learning with pseudo-likelihood functions? (4) Are suitable fragments of relational dependency networks projective \cite{Neville2007}?

\bibliographystyle{aaai}
\bibliography{mplan}

\end{document}